\documentclass[conference]{IEEEtran}
\IEEEoverridecommandlockouts
\usepackage{cite}
\usepackage{amsmath,amssymb,amsfonts}
\usepackage{graphicx}
\usepackage{textcomp}
\usepackage{xcolor}
\def\BibTeX{{\rm B\kern-.05em{\sc i\kern-.025em b}\kern-.08em
    T\kern-.1667em\lower.7ex\hbox{E}\kern-.125emX}}

\usepackage{amsthm}
\usepackage{tabularx}
\usepackage{adjustbox}
\usepackage{array}
\usepackage{booktabs}
\usepackage{algorithm}
\usepackage[noend]{algpseudocode}
\usepackage{csquotes}
\usepackage{url}
\usepackage{tikz}
\usetikzlibrary{arrows.meta, positioning, calc, patterns, decorations.pathreplacing}

\newtheorem{proposition}{Proposition}

\DeclareMathOperator*{\argmax}{arg\,max}
\DeclareMathOperator*{\argmin}{arg\,min}

\begin{document}

\title{An Introduction to Compression-Based Machine Learning}

 \author{
 \IEEEauthorblockN{
 John Hurwitz\IEEEauthorrefmark{1},
 Edward Raff\IEEEauthorrefmark{1}\IEEEauthorrefmark{2}\IEEEauthorrefmark{3},
 Charles Nicholas\IEEEauthorrefmark{1}
 }

 \IEEEauthorblockA{\IEEEauthorrefmark{1}University of Maryland, Baltimore County}
 \IEEEauthorblockA{\IEEEauthorrefmark{2}CrowdStrike}
 \IEEEauthorblockA{\IEEEauthorrefmark{3}Syracuse University}
 }

\maketitle

\begin{abstract}
    Any lossless compression algorithm (like gzip) may be converted into a machine learning method, via either Normalized Compression Distance or the Minimum Description Length principle. Any auto-regressive model may be converted into a lossless compression method via entropy coding. This seemingly circular dependence has unrealized potential in modern artificial intelligence and machine learning, and we survey and formalize the various strategies that have been used to leverage compression for machine learning. We introduce and empirically validate a design framework for compression-based ML, finding compression-based methods competitive with conventional baselines and decisively stronger on malware. We find that varying these design choices yields accuracy gains of up to 0.62.
\end{abstract}


\section{Introduction}

In 1999, Mahoney \cite{mahoney1999text} posited that predicting the next character of a sentence with high accuracy should be the litmus test of AI, as improved prediction requires understanding the text semantically 
(foreshadowing large language models). He also noted that this task is equivalent to compression, and so he studied different compression algorithms like \texttt{gzip} and their ability to approach this next-character prediction target, concluding: 
\begin{displayquote}
\textit{It is somewhat comforting to find that compression, like AI, is still unsolved.}
\end{displayquote}
Despite a quarter-century of progress, it is peculiar that AI is now much closer to being ``solved'' in the eyes of many compared to compression. Today, a niche area of research has blossomed, directly connecting AI/ML and compression, often challenging our assumptions about the source of ML's successes and failures. In this article, we survey these compression-based machine learning methods in light of the remarkable room for improvement and, at the same time, shed light on the insights that might still be extracted from Mahoney's prognostications. 

Compression and prediction have long been recognized as closely related concepts, with MacKay famously noting that information theory and machine learning are ``two sides of the same coin"~\cite{mackay2003information}. At the heart of data compression, the art of reducing the number of bits required to store information, is the prediction of future data~\cite{nelson1995}. This connection has led to a growing body of machine learning literature on the use of compression algorithms for machine learning tasks. To be specific, we are interested in the connections between lossless compression algorithms in the classical computer science sense, like \texttt{gzip}, and machine learning tasks. A broad survey of all inspirations of ``compression'' (e.g., like the hidden state of an auto-encoder) is beyond the scope of this article (lossy compression, which allows less than perfect reconstruction of the input, is also beyond our scope). Kolmogorov complexity~\cite{kolmogorov1963} will be of particular theoretical foundation for our discussion, as it is the length of the shortest program that produces a given input. This is an uncomputable function, but provides the basis of our interest in \textit{information distance} \cite{Bennett1998InformationDistance,vitanyi2008normalizedinformationdistance}, an information-theoretic notion of distance based on \textit{algorithmic} similarity. By replacing the Kolmogorov complexity with an empirical compression algorithm, one obtains a \emph{compression distance}~\cite{li2004similarity}, and along with it, practical methods for leveraging compression in machine learning. 

At a high level, compressed lengths can be interpreted in two distinct ways: either as a measure of similarity between objects, or as a proxy for likelihood under a model. Viewed through the similarity lens, compressed lengths induce a compression distance function that can work for any byte sequence by leveraging the fact that two similar sequences should produce improved compression ratios if considered together. This is the insight behind the seminal Normalized Compression Distance (NCD)~\cite{li2004similarity}. Viewed through the likelihood lens, we obtain another way of using compression for prediction inspired by the Minimum Description Length (MDL)~\cite{rissanen1978} principle. As there is a direct correspondence between code length and probability, a compression model that achieves the shortest code length when compressing a sample is equivalent to the model that maximizes its likelihood. Since compressors operate over bytes, compression-based ML approaches have been especially useful in the cybersecurity domain, where deep learning struggles to learn useful representations for executable files, resulting in various applications for malware classification~\cite{borbely2015normalizedcompressiondistancelarge,raff_lzjd_2017,alshahwan2015detectingmalwareinformationcomplexity,Wehner2005AnalyzingWA,Raff2020a,raff_cybersecurity_2026}.

\textbf{Scope.}
Throughout this survey, we use the term compression in the classical lossless
data compression sense: algorithms that reduce the number of bits required to
represent data while preserving exact reconstruction. Our focus is on how such
compressors induce similarity measures, likelihood estimates, and predictive
models for machine learning. We do not survey model compression techniques
whose goal is to reduce the size or computational cost of machine learning
models themselves, including quantization, pruning, distillation,
sparsification, low-rank approximation, or related efficient-inference methods.
Although these areas also involve compression, they address a fundamentally different question: how to compress a learned model rather than how to perform learning using compression.

To properly position the reader to appreciate the unique connection between data compression and modern ML, we will begin with historical context and key terms in compression from \S \ref{sec:compression_fundamentals}, including how probabilistic predictive models can be turned into a lossless compression algorithm through entropy coding. This connection has only recently been leveraged by modern deep learning to develop better compression algorithms, as we highlight in \S \ref{sec:ml_for_compression}. We then introduce the theoretical foundations of the NCD and the MDL principle in \S \ref{sec:two_lenses}, before unifying existing compression-based ML methods into a common design framework in \S \ref{sec:comp_ml} with empirical validation. We will address current challenges and open questions in \S \ref{sec:challenges}. Finally, we will conclude in \S \ref{sec:conclusion}. 

\section{Compression Fundamentals} \label{sec:compression_fundamentals}
Data compression refers to schemes that try to represent data using fewer bits than were used in the original file. \textit{Lossless} compression refers to such schemes that guarantee perfect reconstruction of the input data. Some examples of lossless compression programs are DEFLATE~\cite{rfc1951} (as utilized in \texttt{gzip} and {PNG}), \texttt{bzip2}, \texttt{LZMA}, and \texttt{zstandard}.

Every compression technique can be viewed as a combination of two distinct stages: \textit{modeling} and \textit{coding}~\cite{nelson1995}. Modeling is the process of identifying regularities in the data, while coding converts the modeled representation into a compact binary representation. There are two main forms of modeling: statistical and dictionary-based. A statistical model yields probabilities for each symbol, and these probabilities can be encoded in a manner where common symbols use fewer bits than do rare symbols~\cite{Cover2009}. Therefore, any model which produces next-symbol probabilities can be converted into a lossless compression scheme by performing \textit{entropy coding} on the probabilities, converting each symbol into a binary representation.
Good compression is achieved when the model produces high as well as accurate probabilities. In contrast, dictionary-based models don't explicitly model per-symbol probabilities; rather, they search for regularity in data by identifying repeated substrings~\cite{nelson1995}. Dictionaries of substrings can either be explicitly or implicitly maintained by keeping a sliding window of past symbols in memory, within which the program can search for matches.

Most popular compression programs nowadays are descendants of LZ77~\cite{ziv1977universal}, the dictionary-based compression algorithm from Lempel and Ziv's seminal paper on sequential data compression. The core idea of the algorithm is to search for backreferences to previously seen substrings in the input stream. Their subsequent work~\cite{ziv1978compression} introduced LZ78, a compression algorithm that explicitly maintained a dictionary of previously seen substrings rather than maintaining a sliding window over a portion of the input stream. Many other techniques for compression exist, though most practical tools use LZ-style algorithms. For a deeper review of data compression techniques, see~\cite{nelson1995,mahoney2013}.

\subsection{Information Theory: Entropy and Kolmogorov complexity}
Data compression has strong theoretical roots in information theory.
Let $\mathcal{X}$ be a discrete alphabet and $p$ a distribution over
it. (Throughout, logarithms are base two and information is measured
in bits.) The \textit{information content} of a symbol
$x\in\mathcal{X}$ is $-\log p(x)$, the ideal code length for that
symbol. The \textit{entropy} of $p$ is the expected information
content of its symbols,
$H(p) := -\sum_{x\in\mathcal{X}} p(x)\log p(x)$, and by Shannon's Source Coding Theorem~\cite{shannon1948} is the
ultimate lossless compression limit for a known distribution.
Encoding data drawn from $p$ with a code optimal for some other
distribution $q$ instead costs the \textit{cross-entropy} $H(p,q) = H(p) + \mathrm{KL}(p\,\|\,q),$
which exceeds the
entropy by the Kullback--Leibler divergence.

\emph{Entropy coding} refers to a class of algorithms which attempt to realize these ideal code lengths in practice, converting a model's probabilities into a bit string whose length approaches $-\log p(x)$ per symbol. Two of the most common entropy coding techniques are Huffman coding~\cite{huf52}, which assigns each symbol a unique integer-length code where common symbols are represented with fewer bits than rare symbols, and arithmetic coding~\cite{arithmetic_coding}, which represents the entire sequence as a single real number of arbitrary precision. Modern neural network-based compression techniques often use arithmetic coding due to it being adaptive, as it can be used when the distribution changes on a per-symbol basis. 

Kolmogorov complexity~\cite{kolmogorov1963} $K(x)$ is the length of the shortest computer program that outputs the input string $x$ and then halts. While entropy gives the ultimate compression limit for probability distributions, $K$ gives such a limit for \emph{individual objects}. %
Conditional Kolmogorov complexity $K(x|y)$ is the length of the shortest program that outputs an input string given some second string $y$ as input, formalizing the minimum amount of information needed to compute one object from another. Though uncomputable, $K$ is upper semi-computable. %
Kolmogorov complexity formalizes the length of an object's shortest effective description, and characterizes when a string is incompressible (a program generating a high-complexity string essentially encodes the string itself). Conditional Kolmogorov complexity inspired the \textit{information distance} $\max\{K(x|y), K(y|x)\}$ and subsequently normalized information distance $\text{NID}(x,y) = \frac{\max\{K(x|y), K(y|x)\}}{\max\{K(x),K(y)\}}$, a metric in $[0,1]$ that theoretically captures ``any effective resemblance between two objects''~\cite{vitanyi2008normalizedinformationdistance}, in the sense that it minorizes all admissible distances~\cite{li2004similarity}. For a deep treatment of Kolmogorov complexity, see~\cite{LiVitanyi2008}.

\section{Machine learning for compression} 
\label{sec:ml_for_compression}
\label{ml_compression}
Neural compression refers to the use of neural networks to perform data compression. Neural compression techniques have in recent years become state-of-the-art lossless compressors in domains such as text~\cite{mahoney_large_text_compression,bellard2021nncp,cmix},  images~\cite{mentzer2020practicalresolutionlearnedlossless,li2024calliccontentadaptivelearning}, and driving video~\cite{commavq} due to the success of neural networks in complex data modeling. The two primary lossless neural compression techniques are those based on \textit{autoregressive models} and those based on \textit{latent variable models}. For a comprehensive introduction to neural compression approaches, see~\cite{yang2023introductionneuraldatacompression}. 

Autoregressive models,  such as Large Language Models (LLMs), predict the next symbol given previous symbols. These probabilities can be encoded via entropy coding, and the better the model, the better the compression. The negative log-likelihood (NLL) of each symbol under the model provides a lower bound on the number of bits required to encode that symbol, illustrating that cross-entropy loss can be interpreted as a compression objective. This fact has resulted in a number of works applying autoregressive generative models directly to lossless text compression~\cite{deletang2024languagemodelingcompression,valmeekam2023llmziplosslesstextcompression,goyal2020dzipimprovedgeneralpurposelossless} and image compression~\cite{oord2016conditionalimagegenerationpixelcnn}. Latent variable models such as Variational Autoencoders (VAEs) can be combined with the ``bits back coding'' technique~\cite{wallace1990classification,frey1996free} to implement a lossless compression scheme.

Compression benchmarks have been introduced with the goal of spurring AI research. The Hutter prize~\cite{Hutter:06hprize} is a challenge to losslessly compress 1GB of Wikipedia to the smallest file possible, under computational constraints. The Large Text Compression Benchmark~\cite{mahoney_large_text_compression} is the same challenge but with relaxed computational constraints. Both benchmarks are led by neural compression techniques, with the top entry of the Hutter prize a variant of CMIX~\cite{cmix}, which uses an ensemble of neural networks to perform bit-level prediction. The top two entries of the Large Text Compression Benchmark are variants of CMIX and NNCP~\cite{bellard2021nncp}, the latter a neural network approach which trains over the sequence in an online fashion and compresses with arithmetic coding. For illustration, as of September 2026,
\texttt{gzip} compresses 1 GB of Wikipedia to approximately 323 MB, while the top entry, \texttt{fx2-cmix-transformer}, compresses to approximately 97 MB~\cite{mahoney_large_text_compression}, including the size of their respective programs. Further, similar neural compression techniques lead the CommaVQ compression challenge to losslessly compress driving video frames tokenized by a VQ-VAE~\cite{oord2018neuraldiscreterepresentationlearning}.

\section{Theoretical Foundations of Compression for Machine Learning: Two Lenses}\label{sec:two_lenses}

We now introduce the theoretical foundations of the compression-as-similarity lens and the compression-as-likelihood lens. Existing compression-based ML approaches differ primarily in how the compressor receives context and the interpretation of compressed lengths for prediction. Similarity approaches treat individual data samples as context for a compression distance calculation, exemplified by NCD. Likelihood approaches learn models from samples and treat those models as context, selecting the model which best compresses the data, exemplified by MDL.

\subsection{Compression as Similarity}
Any lossless compression algorithm induces an information-theoretic compression distance, capturing a notion of algorithmic (dis)similarity. The fundamental insight of compression distance is that when two objects are compressed together, the resulting file size will be smaller if the compressor discovers shared information in one to help reconstruct the other. Let $C(x)$ be the length of the compressed output in bytes when compressing $x$ with some compression algorithm. Then the Normalized Compression Distance (NCD) between two sequences $x$ and $y$ is defined as follows:
\begin{equation}    \label{eq:NCD}
    \textnormal{NCD}(x, y) = \frac{C(xy) - \min \{ C(x), C(y)\}}{\max \{ C(x), C(y)\}}
\end{equation}
NCD is easier to understand if we assume (without loss of generality) that $C(y) \geq C(x)$. Then the NCD becomes $\textnormal{NCD}(x,y) = \frac{C(xy) - C(x)}{C(y)}$. The numerator here represents the extra number of bits required to compress $y$ (the sequence with the larger compressed representation) given the information in $x$ (the shorter compressed representation), effectively testing how well the compressor can reuse information in $x$ to reconstruct $y$. Effective reuse indicates information-theoretic similarity. The denominator is a normalizing factor enforcing a range of $[0,1+\epsilon]$. We provide an illustration of how NCD is used as a similarity measure in Fig.~\ref{fig:ncd_example}.

\begin{figure}[!h]
\centering
\definecolor{fileA}{RGB}{66,133,244}   %
\definecolor{fileB}{RGB}{52,168,83}    %
\definecolor{fileC}{RGB}{234,67,53}    %
\definecolor{savings}{RGB}{46,139,87}
\definecolor{nosavings}{RGB}{180,60,60}

\pgfdeclarepatterninherentlycolored{diagA}{\pgfqpoint{-1pt}{-1pt}}{\pgfqpoint{5pt}{5pt}}{\pgfqpoint{4pt}{4pt}}{
  \pgfsetlinewidth{0.8pt}
  \pgfsetcolor{fileA!70}
  \pgfpathmoveto{\pgfqpoint{-0.5pt}{-0.5pt}}
  \pgfpathlineto{\pgfqpoint{5.5pt}{5.5pt}}
  \pgfusepath{stroke}
}
\pgfdeclarepatterninherentlycolored{diagB}{\pgfqpoint{-1pt}{-1pt}}{\pgfqpoint{5pt}{5pt}}{\pgfqpoint{4pt}{4pt}}{
  \pgfsetlinewidth{0.8pt}
  \pgfsetcolor{fileB!70}
  \pgfpathmoveto{\pgfqpoint{-0.5pt}{-0.5pt}}
  \pgfpathlineto{\pgfqpoint{5.5pt}{5.5pt}}
  \pgfusepath{stroke}
}
\pgfdeclarepatterninherentlycolored{diagC}{\pgfqpoint{-1pt}{-1pt}}{\pgfqpoint{5pt}{5pt}}{\pgfqpoint{4pt}{4pt}}{
  \pgfsetlinewidth{0.8pt}
  \pgfsetcolor{fileC!70}
  \pgfpathmoveto{\pgfqpoint{-0.5pt}{-0.5pt}}
  \pgfpathlineto{\pgfqpoint{5.5pt}{5.5pt}}
  \pgfusepath{stroke}
}

\pgfdeclarepatterninherentlycolored{diagAB}{\pgfqpoint{-1pt}{-1pt}}{\pgfqpoint{9pt}{9pt}}{\pgfqpoint{8pt}{8pt}}{
  \pgfsetlinewidth{0.8pt}
  \pgfsetcolor{fileA!70}
  \pgfpathmoveto{\pgfqpoint{-0.5pt}{-0.5pt}}
  \pgfpathlineto{\pgfqpoint{5.5pt}{5.5pt}}
  \pgfusepath{stroke}
  \pgfsetcolor{fileB!70}
  \pgfpathmoveto{\pgfqpoint{3.5pt}{-0.5pt}}
  \pgfpathlineto{\pgfqpoint{9.5pt}{5.5pt}}
  \pgfusepath{stroke}
}
\pgfdeclarepatterninherentlycolored{diagAC}{\pgfqpoint{-1pt}{-1pt}}{\pgfqpoint{9pt}{9pt}}{\pgfqpoint{8pt}{8pt}}{
  \pgfsetlinewidth{0.8pt}
  \pgfsetcolor{fileA!70}
  \pgfpathmoveto{\pgfqpoint{-0.5pt}{-0.5pt}}
  \pgfpathlineto{\pgfqpoint{5.5pt}{5.5pt}}
  \pgfusepath{stroke}
  \pgfsetcolor{fileC!70}
  \pgfpathmoveto{\pgfqpoint{3.5pt}{-0.5pt}}
  \pgfpathlineto{\pgfqpoint{9.5pt}{5.5pt}}
  \pgfusepath{stroke}
}

\resizebox{\columnwidth}{!}{%
\begin{tikzpicture}[
    file/.style={draw, rounded corners=2pt, minimum height=0.6cm, inner sep=4pt, font=\ttfamily\footnotesize},
    outblock/.style={draw, rounded corners=2pt, minimum height=0.6cm},
    every node/.style={align=center},
]

\node[font=\footnotesize\bfseries, anchor=west] at (0, 0) {Three byte sequences:};

\node[file, fill=fileA!15, draw=fileA!70] (x) at (1.2, -0.8) {\textcolor{fileA!80!black}{ababa}};
\node[font=\scriptsize, anchor=south] at (x.north) {$x$};

\node[file, fill=fileB!15, draw=fileB!70] (y) at (3.2, -0.8) {\textcolor{fileB!80!black}{babab}};
\node[font=\scriptsize, anchor=south] at (y.north) {$y$};

\node[file, fill=fileC!15, draw=fileC!70] (z) at (5.4, -0.8) {\textcolor{fileC!80!black}{abcefg}};
\node[font=\scriptsize, anchor=south] at (z.north) {$z$};

\node[font=\footnotesize\bfseries, anchor=west] at (0, -2.0) {Step 1: Compress each};

\node[font=\scriptsize] (cx_lp) at (0.6, -2.8) {$C($};
\node[file, fill=fileA!15, draw=fileA!70, anchor=west] (cx_in) at (cx_lp.east) {\textcolor{fileA!80!black}{\scriptsize ababa}};
\node[font=\scriptsize, anchor=west] (cx_rp) at (cx_in.east) {$)=$};
\node[outblock, minimum width=0.7cm, pattern=diagA, draw=fileA!70, anchor=west] (cx_out) at (cx_rp.east) {};
\node[font=\scriptsize, anchor=west] (cx_size) at (cx_out.east) {4\,B};

\node[font=\scriptsize] (cy_lp) at (0.6, -3.6) {$C($};
\node[file, fill=fileB!15, draw=fileB!70, anchor=west] (cy_in) at (cy_lp.east) {\textcolor{fileB!80!black}{\scriptsize babab}};
\node[font=\scriptsize, anchor=west] (cy_rp) at (cy_in.east) {$)=$};
\node[outblock, minimum width=0.7cm, pattern=diagB, draw=fileB!70, anchor=west] (cy_out) at (cy_rp.east) {};
\node[font=\scriptsize, anchor=west] (cy_size) at (cy_out.east) {4\,B};

\node[font=\scriptsize] (cz_lp) at (0.6, -4.4) {$C($};
\node[file, fill=fileC!15, draw=fileC!70, anchor=west] (cz_in) at (cz_lp.east) {\textcolor{fileC!80!black}{\scriptsize abcefg}};
\node[font=\scriptsize, anchor=west] (cz_rp) at (cz_in.east) {$)=$};
\node[outblock, minimum width=0.9cm, pattern=diagC, draw=fileC!70, anchor=west] (cz_out) at (cz_rp.east) {};
\node[font=\scriptsize, anchor=west] (cz_size) at (cz_out.east) {5\,B};

\node[font=\scriptsize, text=gray!70!black, text width=2.0cm, anchor=west] (annot) at (5.2, -3.6) {\textit{Compressed size: shorter $=$ more redundancy found}};
\draw[-{Stealth}, gray!60, thin] (annot.west) -- (cx_size.east);
\draw[-{Stealth}, gray!60, thin] (annot.west) -- (cy_size.east);
\draw[-{Stealth}, gray!60, thin] (annot.west) -- (cz_size.east);

\node[font=\footnotesize\bfseries, anchor=west] at (0, -5.5) {Step 2: Compress concatenations};

\node[font=\scriptsize] (cxy_lp) at (0.6, -6.4) {$C($};
\node[file, fill=fileA!15, draw=fileA!70, inner sep=2pt, anchor=west] (cxy_in1) at (cxy_lp.east) {\textcolor{fileA!80!black}{\scriptsize ababa}};
\node[file, fill=fileB!15, draw=fileB!70, inner sep=2pt, anchor=west] (cxy_in2) at ([xshift=-0.4pt]cxy_in1.east) {\textcolor{fileB!80!black}{\scriptsize babab}};
\node[font=\scriptsize, anchor=west] (cxy_rp) at (cxy_in2.east) {$)=$};
\node[outblock, minimum width=0.9cm, pattern=diagAB, draw=gray!70, anchor=west] (cxy_out) at (cxy_rp.east) {};
\node[font=\scriptsize, text=savings, anchor=west] at (cxy_out.east) {5\,B \textit{(shared!)}};

\node[font=\scriptsize] (cxz_lp) at (0.6, -7.3) {$C($};
\node[file, fill=fileA!15, draw=fileA!70, inner sep=2pt, anchor=west] (cxz_in1) at (cxz_lp.east) {\textcolor{fileA!80!black}{\scriptsize ababa}};
\node[file, fill=fileC!15, draw=fileC!70, inner sep=2pt, anchor=west] (cxz_in2) at ([xshift=-0.4pt]cxz_in1.east) {\textcolor{fileC!80!black}{\scriptsize abcefg}};
\node[font=\scriptsize, anchor=west] (cxz_rp) at (cxz_in2.east) {$)=$};
\node[outblock, minimum width=1.5cm, pattern=diagAC, draw=gray!70, anchor=west] (cxz_out) at (cxz_rp.east) {};
\node[font=\scriptsize, text=nosavings, anchor=west] at (cxz_out.east) {8\,B \textit{(no help)}};

\node[font=\footnotesize\bfseries, anchor=west] at (0, -8.4) {Step 3: Compute NCD};

\node[anchor=north west, font=\scriptsize] at (0.3, -9.0) {%
$\text{NCD}(x,y) = \dfrac{C(xy) - \min\{C(x), C(y)\}}{\max\{C(x), C(y)\}} = \dfrac{5 - 4}{4} = \textcolor{savings}{\mathbf{0.25}}$
};

\node[anchor=north west, font=\scriptsize] at (0.3, -10.2) {%
$\text{NCD}(x,z) = \dfrac{C(xz) - \min\{C(x), C(z)\}}{\max\{C(x), C(z)\}} = \dfrac{8 - 4}{5} = \textcolor{nosavings}{\mathbf{0.80}}$
};

\draw[thick, gray!50] (0, -11.3) -- (7.0, -11.3);
\node[anchor=north west, font=\scriptsize] at (0, -11.4) {%
$\textcolor{savings}{0.25} \ll \textcolor{nosavings}{0.80}$\; $\Rightarrow$ compressor confirms $x$ and $y$ are more similar.
};

\end{tikzpicture}%
}
\caption{Illustration of NCD between three byte sequences. Sequences $x$ and $y$ share repeating sub-patterns, so their joint compression $C(xy)$ yields a small NCD. Sequence $z$ shares little structure with $x$, resulting in a large NCD. Output block widths represent compressed sizes; diagonal stripes indicate the source sequences.}
\label{fig:ncd_example}
\end{figure}

 NCD is a realizable version of information distance which approximates Kolmogorov complexity using real compressors. For a deeper treatment of Kolmogorov complexity, information distance, compression distance, and their applications, see~\cite{LiVitanyi2008}.

Compression-based ML has long shown a historical parallel to deep learning in that it eschews handcrafted feature engineering, allowing ``the algorithm'' to ``figure it out''. This compression-as-learning has been most widely applied to text categorization~\cite{wan2024,jiang-etal-2023-low,bratko2006spam,marton2005compression,li2004similarity,cilibrasi2005clustering,Benedetto_2002,eibeTextCategorization}, malware detection~\cite{raff2020new,Raff2017MalwareCA,raff_lzjd_2017,borbely2015normalizedcompressiondistancelarge,alshahwan2015detectingmalwareinformationcomplexity,Wehner2005AnalyzingWA}, DNA sequence clustering~\cite{li2004similarity}, and data mining and anomaly detection~\cite{keogh2004,keogh2007}. A correspondence between compression distance approaches and traditional feature vector approaches has been established~\cite{sculley2006}, with the mechanism of the compressor defining an implicit feature space.

There exist variants of NCD that capture the same underlying notion of similarity. The Chen-Li Metric reads $\textnormal{CLM}(x,y) = 1 - \frac{C(x) - C(x|y)}{C(xy)}$; \cite{keogh2004} defines the \textit{Compression-based Dissimilarity Measure} (CDM):
$\textnormal{CDM}(x,y) = \frac{C(xy)}{C(x) + C(y)}$;
and \cite{sculley2006} defines \textit{compression-based cosine}:
$\text{CosS}(x,y) = 1 - \frac{C(x) + C(y) - C(xy)}{\sqrt{C(x)C(y)}}$.
NCD, CDM, and CosS have all been shown to reduce to a canonical 
form~\cite{sculley2006}:
$1 - \frac{C(x) + C(y) - C(xy)}{f(x,y)}$ where $f(x,y)$ is a normalizing term. This helps to explain the highly similar experimental results regardless of the particular choice of these compression distance formulations~\cite{sculley2006,jiang_few-shot_2022}.

\subsection{Compression as Likelihood}

Machine learning is in the business of inductive inference, and the Minimum Description Length (MDL) principle~\cite{rissanen1978} proposes a solution to the model selection problem through the lens of compression. In MDL, the best model is the one that best compresses the data, considering the size of the model itself. Often described as a formalization of Occam's Razor, MDL explicitly ties learning to compression. The minimization of the size of the encoding of the model plus the size of the encoding of the data, given the model, coincides with a natural take on practical data compression as balancing the compressed representation with the size of the compression program. The MDL principle implies at least one way that learned compression models can yield classifications on data in an ML context: by learning a separate model per class, the class is chosen whose model best compresses the test sample. 

Formally, given a dataset $D$ and a set of candidate hypotheses $\mathcal{H}$, the MDL principle selects the hypothesis $H\in\mathcal{H}$ that minimizes the total description length $L(H) + L(D|H)$ where $L(H)$ represents the length of the description of the model, and $L(D|H)$ is the length of the data when encoded with the help of the model. Using the well-known correspondence between probability and code lengths that an event with probability $p$ can be encoded in $-\log_2p$ bits, we see that the model which minimizes description length is equivalently the maximum a posteriori (MAP) estimate from Bayesian inference. Therefore both the MDL and the MAP prediction rule select the hypothesis which maximizes the posterior probability of the hypothesis given the data: 
\begin{equation}
    \argmin_H  L(H) + L(D|H) = \argmax_H P(H|D)
\end{equation}

\begin{figure}[!h]
\centering
\resizebox{\columnwidth}{!}{\definecolor{class1}{RGB}{66,133,244}   %
\definecolor{class2}{RGB}{245,140,30}   %
\definecolor{testc}{RGB}{110,110,120}   %
\definecolor{cheap}{RGB}{46,139,87}     %
\definecolor{costly}{RGB}{180,60,60}    %

\pgfdeclarepatterninherentlycolored{diag1}{\pgfqpoint{-1pt}{-1pt}}{\pgfqpoint{5pt}{5pt}}{\pgfqpoint{4pt}{4pt}}{
  \pgfsetlinewidth{0.8pt}\pgfsetcolor{class1!70}
  \pgfpathmoveto{\pgfqpoint{-0.5pt}{-0.5pt}}\pgfpathlineto{\pgfqpoint{5.5pt}{5.5pt}}\pgfusepath{stroke}
}
\pgfdeclarepatterninherentlycolored{diag2}{\pgfqpoint{-1pt}{-1pt}}{\pgfqpoint{5pt}{5pt}}{\pgfqpoint{4pt}{4pt}}{
  \pgfsetlinewidth{0.8pt}\pgfsetcolor{class2!75}
  \pgfpathmoveto{\pgfqpoint{-0.5pt}{-0.5pt}}\pgfpathlineto{\pgfqpoint{5.5pt}{5.5pt}}\pgfusepath{stroke}
}

\begin{tikzpicture}[
    file/.style={draw, rounded corners=2pt, minimum height=0.7cm, inner sep=5pt, font=\ttfamily\small},
    outblock/.style={draw, rounded corners=2pt, minimum height=0.7cm},
    seg/.style={draw, rounded corners=1pt, minimum height=0.7cm, inner sep=0pt},
    label/.style={font=\small},
    every node/.style={align=center},
]

\node[label, anchor=south, font=\bfseries] at (0.7, 8.3) {Test sample $x$, unknown class};
\node[file, fill=testc!12, draw=testc, minimum width=1.7cm] (xtest) at (0.7, 7.7) {\textcolor{testc!85!black}{ababab}};

\node[font=\bfseries, text=class1!80!black] at (-2.2, 6.7) {Class 1};
\node[font=\bfseries, text=class2!70!black] at ( 3.6, 6.7) {Class 2};

\node[font=\bfseries, anchor=west] at (-6.4, 6.1) {Step 1: build class models};
\node[file, fill=class1!15, draw=class1!70, minimum width=2.2cm] (b1) at (-2.2, 5.2) {\textcolor{class1!80!black}{abababab}};
\node[label, anchor=south] at (b1.north) {$B_1$};
\node[file, fill=class2!18, draw=class2!75, minimum width=2.2cm] (b2) at ( 3.6, 5.2) {\textcolor{class2!65!black}{xyxyxyxy}};
\node[label, anchor=south] at (b2.north) {$B_2$};

\node[font=\bfseries, anchor=west] at (-6.4, 4.1) {Step 2: compress each model};
\node[outblock, minimum width=2.0cm, pattern=diag1, draw=class1!70, anchor=west] (cb1) at (-3.2, 3.5) {};
\node[label, anchor=north] at (cb1.south) {$C(B_1)=10$ B};
\node[outblock, minimum width=2.0cm, pattern=diag2, draw=class2!75, anchor=west] (cb2) at ( 2.6, 3.5) {};
\node[label, anchor=north] at (cb2.south) {$C(B_2)=10$ B};

\node[font=\bfseries, anchor=west] at (-6.4, 2.2) {Step 3: append $x$, recompress};
\node[outblock, minimum width=2.0cm, pattern=diag1, draw=class1!70, anchor=west] (cb1x) at (-3.2, 1.6) {};
\node[seg, minimum width=0.2cm, fill=cheap, draw=cheap!60!black, anchor=west] (d1) at (cb1x.east) {};
\node[label, anchor=north] at ([xshift=0.1cm]cb1x.south) {$C(B_1 x)=11$ B};
\draw[decorate, decoration={brace, amplitude=4pt}, cheap!55!black]
  ([yshift=2pt]d1.north west) -- ([yshift=2pt]d1.north east)
  node[midway, above=3pt, font=\scriptsize, text=cheap!55!black] {$\Delta_1$};
\node[outblock, minimum width=2.0cm, pattern=diag2, draw=class2!75, anchor=west] (cb2x) at (2.6, 1.6) {};
\node[seg, minimum width=1.2cm, fill=costly, draw=costly!60!black, anchor=west] (d2) at (cb2x.east) {};
\node[label, anchor=north] at ([xshift=0.3cm]cb2x.south) {$C(B_2 x)=16$ B};
\draw[decorate, decoration={brace, amplitude=4pt}, costly!55!black]
  ([yshift=2pt]d2.north west) -- ([yshift=2pt]d2.north east)
  node[midway, above=3pt, font=\scriptsize, text=costly!55!black] {$\Delta_2$};

\node[font=\bfseries, anchor=west] at (-6.4, 0.1) {Step 4: conditional code length};
\node[anchor=west, font=\small] at (-4.5, -0.6)
  {$\Delta_1 = C(B_1 x) - C(B_1) = 11 - 10 = \textcolor{cheap}{\mathbf{1}}$ B};
\node[anchor=west, font=\small] at (-4.5, -1.3)
  {$\Delta_2 = C(B_2 x) - C(B_2) = 16 - 10 = \textcolor{costly}{\mathbf{6}}$ B};

\draw[thick, gray] (-6.4, -2.2) -- (6.2, -2.2);
\node[anchor=north, font=\small] at (-0.1, -2.4)
  {$\hat{y} = \arg\min_i \Delta_i$: since $\Delta_1=\textcolor{cheap}{1} \ll \textcolor{costly}{6}=\Delta_2$, predict Class 1.};

\end{tikzpicture}}
\caption{MDL classification (compression as likelihood). Two classes, each represented by a model $B_i$ built from its training data, are compressed alone and then with the test sample $x$ appended. The conditional code length $\Delta_i = C(B_i x) - C(B_i)$ is the extra cost of encoding $x$ under each model; $x$ is assigned to the class minimizing it. Output block widths represent compressed sizes; the solid segment is $\Delta_i$.}
\label{fig:mdl_example}
\end{figure}
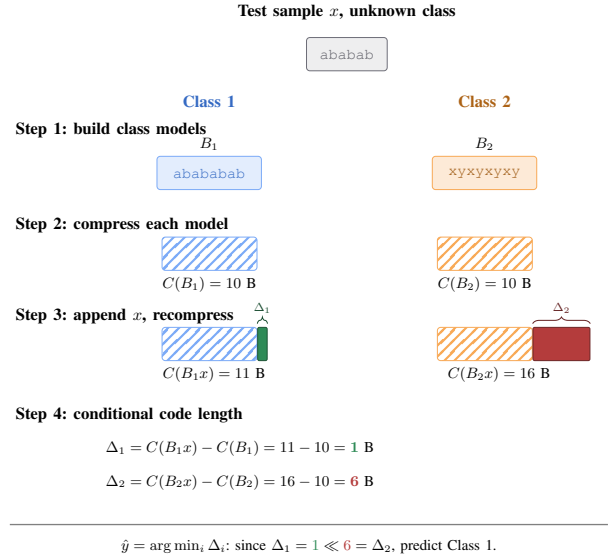

\subsection{When NCD and MDL coincide}
Though a rigorous analysis of when a compression-distance approach or an MDL approach is preferable remains an open challenge, we can make a simple but unifying observation: Both NCD and MDL ultimately rank candidates by how effectively a compressor can use side information to encode a test sample. In NCD, this side information is another data sample, while in MDL it is a model. Under a uniformity condition on the complexity of the side information, these two viewpoints induce the same ranking over candidates, allowing NCD to be interpreted as an MDL-style conditional code length in which data points act as implicit models, and allowing MDL to be interpreted as selecting the model which minimizes its NCD to the test point under the same uniformity condition. In particular, the context point minimizing the MDL objective also minimizes NCD.
\begin{proposition}
Let $\Sigma$ be a set of symbols, let $A: \Sigma^*\times\Sigma^*\to\Sigma^*$ be an aggregation operation, let $x$ be a fixed test sample, and let $\{s_i\}$ be a set of context objects. Let $C(\cdot)$ denote compressed length under a fixed compressor. Assume $C(s_i) = L$ for all $i$ (uniform compressed context length). Then
\[
\arg\min_i [C(A(s_i,x)) - C(s_i)] \;=\; \arg\min_i \text{NCD}(s_i,x).
\]
\end{proposition} 
\begin{proof}
The left hand side is equivalent to $\arg\min_iC(A(s_i,x))$ as $C(s_i)$ is constant. Expanding the right hand side, $\arg\min_i \text{NCD}(s_i,x) = \arg\min_i \frac{C(A(s_i,x)) - \min\{C(s_i),C(x)\}}{\max\{C(s_i),C(x)\}} = \arg\min_i C(A(s_i,x))$ as $C(x)$ is also constant.
\end{proof}
When $C(s_i)$ is not uniform, MDL and NCD differ in how they account for context complexity: MDL applies an explicit penalty, while NCD applies normalization. We provide an illustration of how the MDL rule is used to classify a sample given per-class compression models in Fig.~\ref{fig:mdl_example}.

\subsection{When is compression-based similarity meaningful?}
As the theoretical justification for compression distance is an approximation to conditional Kolmogorov complexity, the mechanism and quality of a compressor affect the quality of the induced compression distance. Below, we crystallize some previously informal considerations of the compression algorithm that affect NCD.

\textbf{Appropriate compressor for data:} Clearly, the chosen compressor should be good at compressing the data at hand in order to define an effective distance. 

\textbf{Window size:} For the NCD to be meaningful, the compressor must be able to identify shared information across both sequences. As traditional compressors tend to have a fixed-size window (for example, \texttt{gzip} usually operates with a 32KB sliding window), the aggregated sequence $xy$ should fit within this window. NCD has been empirically shown to increase towards 1 when the sequence does not fit within the compressor's window \cite{Cebrian2005NCDPitfalls}.

\textbf{Effective compression of $xy$:} NCD not only requires the compressor to be good at compressing individual samples, but it should also be good at compressing the aggregated sequence $xy$. The extent to which the compressor improves its ability to compress $y$ when given $x$'s information is crucial to defining a meaningful NCD metric. As a warning, in neural compression, LLMs compress text well but yield unpredictable results when performing NCD-based classification, with a likely explanation that compressed concatenations result in out-of-distribution sequences, which worsens the model's ability to identify semantic similarity~\cite{hurwitz2024neuralnormalizedcompressiondistance,hurwitz2025llm_ncd}.

Given the superior compression ability of neural compressors, the natural question arises: \textit{do the improved compression ratios of neural compressors translate into improved predictive performance via compression distance?} This is a new area of research, and there is much that is still unknown. The first published work in the area of neural compressors for NCD-based classification uses VAE neural compressors for few-shot image classification~\cite{jiang_few-shot_2022}, showing an improvement in accuracy over traditional compressors. Another work applied LLMs to neural text compression for few-shot text classification \cite{hurwitz2024neuralnormalizedcompressiondistance}, presenting counterintuitive experimental evidence of a lack of correlation between compression ratio and NCD accuracy. This is likely due to the different mechanism of the $C(xy)$ term when using neural compressors versus traditional compressors. When using an LLM to compress $xy$, shared information is successfully utilized if the probabilities assigned to the tokens in $y$ are higher than they would be when considered alone, as a result of having previously processed the tokens in $x$. There is much work yet to be done in the application of neural compressors to compression-based ML.

\section{Compression-Based Machine Learning} \label{sec:comp_ml}
Having described the two paradigms of compression-as-similarity and compression-as-likelihood, we now formalize that much of the compression-based ML literature can be understood as making a small number of design choices. While methods differ in whether they rely on pairwise similarity or compression models, they share a common structure: a compressor induces code lengths, portions of training data are made available to the compressor as context, a scalar measure is derived from compressed lengths, and a prediction rule maps these quantities to outputs.  Our insight is that many prior works can be unified into a broader framework that characterizes these design choices which we describe as four factors: Compression algorithm, Context, Compression measure, and Prediction technique.  Examples of prior works encoded in this framework are given in Table \ref{tab:compression-ml}. 

\begin{table*}[t!]
\caption{Papers in compression-based ML and their settings under our framework.}
\centering
\setlength{\tabcolsep}{3pt} %
\renewcommand{\arraystretch}{1.1} %

\begin{tabularx}{\textwidth}{@{}X p{2.8cm} X X X X@{}}
\toprule
\textbf{Work} & \textbf{Compressor} & \textbf{Measure} & \textbf{Context Source} & \textbf{Context Injection Mechanism} & \textbf{Prediction} \\
\midrule
\cite{eibeTextCategorization} & PPM & conditional code length & probabilistic model & model-conditioned encoding & MDL \\
\cite{li2001} & domain-specific (DNA) & NCD (variant) & single sample & string concat & tree-building \\
\cite{Benedetto_2002} & gzip & conditional code length & aggregated class data & string concat & 1-NN, tree-building \\
\cite{teahan2001compression} & PPM & conditional code length & probabilistic model & model-conditioned encoding & MDL \\
\cite{li2004similarity} & gzip, domain-specific (DNA) & NCD & single sample & string concat & tree-building \\
\cite{keogh2004} & zip & CDM & single sample & string concat & 1-NN, clustering \\
\cite{chen_shared_2004} & domain-specific (TokenCompress) & NCD & single sample & string concat & similarity score\\
\cite{cilibrasi2005clustering} & gzip, bzip2, PPMZ & NCD & single sample & string concat & tree-building, clustering \\
\cite{marton2005compression} & RAR, gzip, LZW & NCD, conditional code length & single sample, aggregated class data & string concat & MDL, 1-NN \\
\cite{Wehner2005AnalyzingWA} & bzip2, zlib & NCD & single sample & string concat & tree-building, 1-NN \\
\cite{sculley2006} & PPM, LZW, LZ77 & NCD, CosS, CDM, CLM & single sample & string concat & k-NN\\
\cite{bratko2006spam} & PPM, DMC & conditional code length & probabilistic model & model-conditioned encoding & MDL \\
\cite{borbely2015normalizedcompressiondistancelarge} & lzma, bzip2, zlib, PPMZ & NCD & single sample & interleaving, similarity chunking & 1-NN \\
\cite{raff_lzjd_2017} & lzset & LZJD & single sample (representation) & Jaccard index & k-NN \\
\cite{raff2020new} & BWT + Markov model & BWMD & single sample (representation) & feature vector & k-NN \\
\cite{kasturi2021} & zstandard & conditional code length (ensemble) & trained zstandard dictionary & model-conditioned encoding & Gradient-Boosted Decision Tree w/ code length features\\
\cite{jiang_few-shot_2022} & VAE + bits-back coding & NCD & single sample & pixel-wise aggregation & k-NN \\
\cite{jiang-etal-2023-low} & gzip & NCD & single sample & string concat & k-NN \\
\cite{hurwitz2025llm_ncd} & LLMs (e.g. Llama3, GPT-2)& NCD & single sample & string concat & k-NN \\
\bottomrule
\end{tabularx}
\label{tab:compression-ml}
\end{table*}

We describe compression-based ML in terms of four primary design choices: a compression algorithm inducing a compressed length function $C$, a context $S$ derived from the training data which %
provides the compressor with side information, a compression measure $M$ defined over code lengths, and a prediction technique $P$. Let $\Sigma$ be a finite alphabet and $\Sigma^*$ the set of all finite strings over $\Sigma$. Let the labeled training set be $\mathcal{T} = \{(z_i,y_i)\}_{i=1}^{N}$ with $z_i \in \Sigma^*, y_i \in \mathcal{Y}$. For each class $y \in \mathcal{Y}$, let $D_y$ denote its training data. The design choices constituting this compression-based ML framework are as follows.

\textbf{Compression algorithm ($C$):} Determines the length of compressed output and thus directly affects the induced distance or code length. Given a particular compression algorithm, and optionally a class-specific compression model $m$, we define a compressed-length function:
$C_m: \Sigma^* \to \mathbb{Z}$.

\textbf{Context ($S$):} The side information which is provided to the compressor prior to encoding a test sample and affects the output compressed length. Depending on the method, context may take the form of another sample, an aggregation of multiple samples according to some aggregation operation, or an explicit learned compression model. To cleanly separate the notions of the side information itself from how it is exposed to the compressor, we refer to the side information as the context source, while we refer to the mechanism by which this side information is exposed to the compressor as the context injection mechanism. 

For clarity, consider $\text{NCD}(y,x)$ between a test sample $x$ and train sample $y$. The context source is $y$, while the context injection mechanism is the aggregation of $x$ and $y$, typically sequence concatenation, resulting in the compressor first compressing one then compressing the other with the benefit of whatever patterns the compressor has already discovered from the first sequence. For MDL classification methods, the context is viewed as a \textit{model}. This can either be an \textit{explicit} model constructed via an algorithm, or an \textit{implicit} model which relies on the compressor's state after compressing the contextual source data. The context source is the learned compression model, while the context injection mechanism describes how the compressor uses that model to compress new data.

\textbf{Compression measure ($M$):} A scalar derived from compressed lengths that serves as the basis for prediction. This may be a compression ratio, a normalized distance (e.g. NCD, LZJD), a conditional code length, or another code-length-based expression. A compression measure maps $r$ compressed lengths to a scalar:
$M: \mathbb{Z}^r \to \mathbb{R}$.

\textbf{Prediction technique ($P$):} Maps compression measures to a prediction. For distance-based measures, this typically involves $k$-NN or clustering. For code-length-based measures, prediction is typically made by selecting the class whose compression model best optimizes the compression measure. Prediction maps $n$ real-valued inputs to a label:
$P: \mathbb{R}^n \to \mathcal{Y}$.

\begin{figure*}[t!]
\centering
\resizebox{0.9\textwidth}{!}{\definecolor{simcol}{RGB}{66,133,244}   %
\definecolor{likcol}{RGB}{245,140,30}   %
\definecolor{slate}{RGB}{70,80,95}
\begin{tikzpicture}[
    header/.style={rounded corners=3pt, fill=slate, text=white,
                   font=\bfseries\footnotesize, minimum width=2.7cm,
                   minimum height=0.6cm, align=center, inner sep=3pt},
    chip/.style={rounded corners=3pt, draw=gray!45, fill=gray!8,
                 text width=2.4cm, align=center, font=\scriptsize,
                 minimum height=0.85cm, inner sep=3pt},
    simchip/.style={chip, draw=simcol, line width=0.9pt, fill=simcol!8},
    likchip/.style={chip, draw=likcol, line width=0.9pt, fill=likcol!10},
    flowsim/.style={-{Stealth[length=2.4mm]}, simcol, line width=1.6pt, opacity=0.85},
    flowlik/.style={-{Stealth[length=2.4mm]}, likcol, line width=1.6pt, opacity=0.85},
    pipe/.style={-{Stealth[length=2mm]}, gray!55, line width=1pt},
    every node/.style={align=center},
]
\node[header] (hc) at (0.0, 4.3) {$C$: Compressor};
\node[header] (hs) at (3.1, 4.3) {$S$: Context};
\node[header] (hm) at (6.2, 4.3) {$M$: Measure};
\node[header] (hp) at (9.3, 4.3) {$P$: Prediction};
\draw[pipe] (hc.east) -- (hs.west);
\draw[pipe] (hs.east) -- (hm.west);
\draw[pipe] (hm.east) -- (hp.west);
\node[simchip] (c1) at (0.0, 3.15) {\texttt{gzip}, \texttt{bzip2}, \texttt{LZMA},...};
\node[likchip] (c2) at (0.0, 2.05) {\texttt{zstandard (dictionary)}};
\node[chip]    (c3) at (0.0, 0.95) {PPM, DMC};
\node[chip]    (c4) at (0.0,-0.15) {LZJD, BWMD\\\textit{(extract representation)}};
\node[chip]    (c5) at (0.0,-1.25) {VAE, LLM + entropy coding\\\textit{(neural)}};
\node[simchip] (s1) at (3.1, 3.15) {single sample\\\textit{concat / interleave}};
\node[likchip] (s2) at (3.1, 2.05) {aggregated\\class data};
\node[chip]    (s3) at (3.1, 0.95) {per-class learned model\\\textit{compression dictionary / compressor state}};
\node[simchip] (m1) at (6.2, 3.15) {NCD, CDM, CosS};
\node[chip]    (m2) at (6.2, 2.05) {LZJD distance};
\node[likchip] (m3) at (6.2, 0.95) {conditional\ code length\\$C(B_ix){-}C(B_i)$};
\node[chip]    (m4) at (6.2,-0.15) {$C_{m_i}(x)$};
\node[simchip] (p1) at (9.3, 3.15) {$k$-NN};
\node[chip]    (p2) at (9.3, 2.05) {clustering / tree};
\node[likchip] (p3) at (9.3, 0.95) {MDL $\arg\min$};
\node[chip]    (p4) at (9.3,-0.15) {use compression ratio features with traditional ML techniques};
\draw[flowsim] (c1.east) -- (s1.west);
\draw[flowsim] (s1.east) -- (m1.west);
\draw[flowsim] (m1.east) -- (p1.west);
\draw[flowlik] (c2.east) -- (s2.west);
\draw[flowlik] (s2.east) -- (m3.west);
\draw[flowlik] (m3.east) -- (p3.west);
\draw[flowsim] (-1.25,-2.35) -- (-0.45,-2.35);
\node[anchor=west, font=\scriptsize] at (-0.35,-2.35)
  {Compression as similarity, e.g.\ NCD with $k$-NN};
\draw[flowlik] (-1.25,-2.90) -- (-0.45,-2.90);
\node[anchor=west, font=\scriptsize] at (-0.35,-2.90)
  {Compression as likelihood, e.g.\ MDL classification};
\end{tikzpicture}}
\caption{The compression-based ML design space. Every method in Table~\ref{tab:compression-ml} amounts to selecting one option per stage of the pipeline $C \to S \to M \to P$. Compression-as-similarity (blue) picks a single-sample context, a compression distance, and a distance-based predictor; compression-as-likelihood (orange) picks a learned class model such as a compression dictionary or compressor state, a conditional code length, and the MDL rule. }
\label{fig:design-space}
\end{figure*}
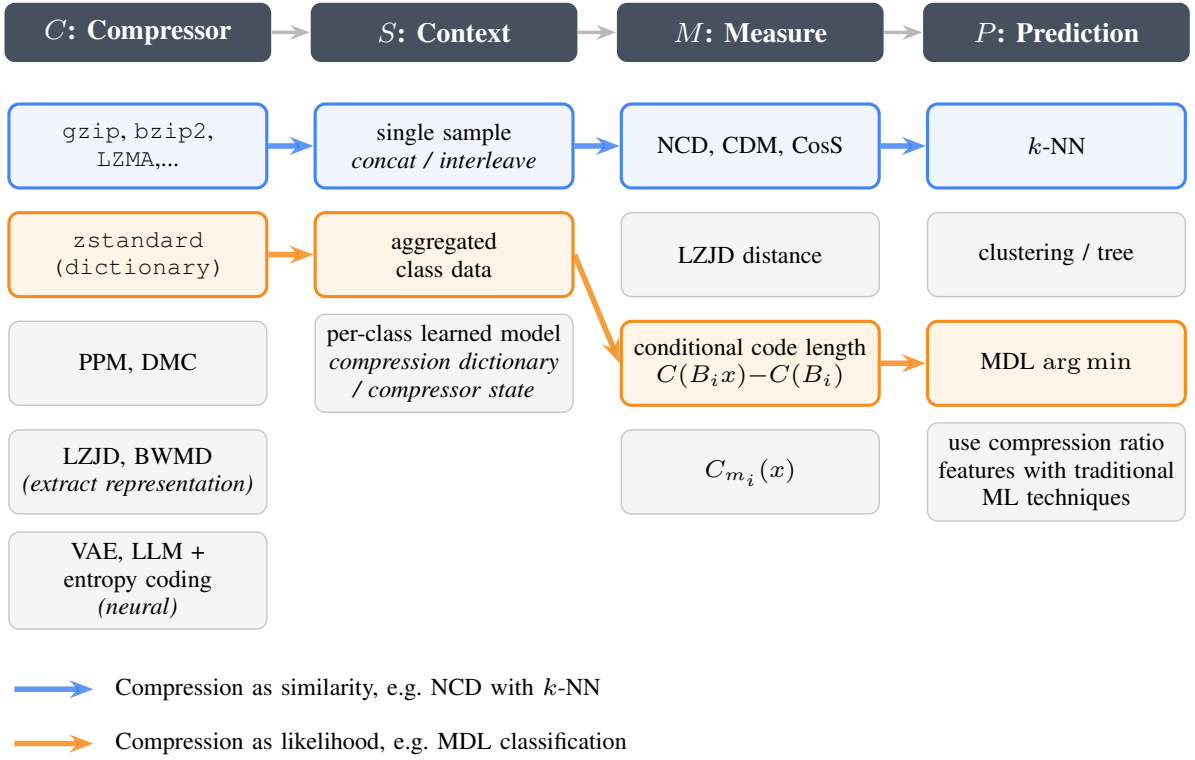

We note an additional hyperparameter called compressors per class (CPC), utilized in~\cite{ftcc}, which partitions each class into a number of subsets via the aggregation operation. The CPC is the number of subsets per class. Each subset is now viewed as a single training sample, which is combined with the test point at prediction time. When CPC is equal to the number of samples in the class, each original training sample is treated individually and the full pairwise compression distance matrix is computed between every train point and test point; the computational complexity of this is $O(n^2)$. When CPC is $1$, each class becomes a single point by combining all of its constituent points via the aggregation operation. This changes the computational complexity to $O(n|\mathcal{Y}|)$, a major improvement given a relatively small number of classes $|\mathcal{Y}|$. More generally, each class can be partitioned into an arbitrary number of subsets in order to facilitate chunking of class data. 

There are arguably two natural ways that one may attempt classification with compressors, and indeed some of the earliest works in compression-based ML pursue either one or the other, with only one work that we are aware of explicitly comparing the two \cite{marton2005compression}. The first way is to recognize that a pairwise compression distance matrix allows the use of distance-based algorithms such as $k$-nearest neighbors ($k$-NN) and clustering. The second way is to obtain per-class compression models and select the class of the model that best compresses a test point.  We thus argue that compression-based ML approaches fall broadly into these two categories: compression-distance-based techniques~\cite{li2001,Benedetto_2002,cilibrasi2005clustering,li2004similarity,chen_shared_2004,keogh2004,marton2005compression,keogh2007,Wehner2005AnalyzingWA,borbely2015normalizedcompressiondistancelarge,raff_lzjd_2017,raff2020new,jiang_few-shot_2022,jiang-etal-2023-low}, which usually normalize compressed lengths to obtain a distance metric (e.g. NCD), though some instead derive distances from intermediate representations of the compressor~\cite{raff_lzjd_2017,raff2020new}, and MDL-inspired techniques~\cite{eibeTextCategorization,teahan2001compression,marton2005compression,ftcc}, which treat compressed lengths as code lengths and classify by selecting the class model yielding the shortest code length and thus assigning the maximum likelihood. Both categories rely on compressor outputs, but differ in whether they interpret them as distances or as proxies for likelihood.

\subsection{Compression-distance-based prediction}
\label{sec:compression_distance_prediction}
Given a compressor and classification task, one can perform compression-based classification by computing the compression distance matrix between every pair of test points and train points, and using these distances for $k$-nearest neighbors. Given a compressor, $C$ is the induced code length function, the context $S$ is the second point injected typically via sequence concatenation, and $M$ is a chosen compression distance function such as NCD. Prediction $P$ is then applied to the resulting distances (e.g., via $k$-NN). In compression-distance-based prediction, context injection is achieved by an aggregation operation that combines the test sample with the context sample: $A: \Sigma^* \times \Sigma^* \rightarrow \Sigma^*$. Throughout the NCD literature, aggregation is most often ordinary string concatenation. The ultimate purpose of $A$ is to integrate the information from two sequences $x$ and $y$ in a manner which allows a good approximation of the information distance.

For a test sample $x \in \Sigma^*$ , training sample $z \in \mathcal{T}$, and an aggregation operation $A$, our compression measure $M$ is the NCD. For ease of reading, we introduce the short hand notation: let $C_x = C(x)$, $C_z = C(z)$, and $C_{zx} = C\!\big(A(z,x)\big)$.
\begin{align*}
M_{\text{NCD}}(C_x, C_z, C_{zx})
  &= \frac{C_{zx} - \min\{C_z, C_x\}}{\max\{C_z, C_x\}}.
\end{align*}
Collecting over all training samples gives 
$d(x) = \big(d(z_i,x)\big)_{i=1}^N \in \mathbb{R}^N,
$
and our prediction $\hat{y}$ is then given by 
$\hat{y} = P(d(x)).$

As the $C(x)$ and $C(y)$ terms are reused across pairs, these can be precomputed in time linear in the number of samples (in the case where the aggregation operation is string concatenation). The complexity of the pairwise approach is dominated by the computation of the distance matrix, which requires $\mathcal{O}(n^2)$ compressor calls.

We note that in the context of sequential data compression and a string concatenation aggregation, there is redundant computation in the naive construction of the pairwise NCD matrix. For all $x$ and $y$, we must calculate $C(xy)$. For identical $x$, the state of the compressor after compressing the $x$ portion is identical no matter the $y$. For fast compressors and low data regimes, this may not be a hindrance. However, where computational efficiency is a concern, an approach that is able to cache $x$'s compression dictionary and initialize a compressor with it will work around this redundancy. %

\begin{algorithm}[!h]
\caption{NCD-KNN Classification}
\label{alg:ncd-knn}
\begin{algorithmic}
\State Precompute $C(y_j)$ for all training samples
\State Precompute $C(x_i)$ for all test samples

\For{each test sample $x_i$}
    \For{each training sample $y_j$}
        \State Compute $C(x_i y_j)$
        \State Compute NCD and store in distance matrix:
        \[
        D[i][j] = \frac{C(x_i y_j) - \min(C(x_i), C(y_j))}{\max(C(x_i), C(y_j))}
        \]
    \EndFor
\EndFor 

\State Apply $k$-nearest neighbors using distance matrix $D$
\end{algorithmic}
\end{algorithm}

Though it is most typical for the aggregation operation to be conventional string concatenation, aggregation generalizes any operation that combines $x$ and $y$ into a new representation. In the image domain, pixel-wise operations have been used to combine images for NCD-based $k$-NN classification~\cite{jiang_few-shot_2022}. In malware detection on large files, interleaving/rearranging chunks of each file has been used to place the similar portions close to each other, allowing for more meaningful NCDs when simple concatenation would render the compressor's sliding window too small to identify similarity across files~\cite{borbely2015normalizedcompressiondistancelarge}.
Another alternative to the high cost of $k$-NN style NCD is to somehow ``featurize'' the result into a vector.~\cite{alshahwan2015detectingmalwareinformationcomplexity} used a randomly chosen reference set of exemplars from each class and created a feature vector for each sample by calculating the NCDs between the sample and each reference point. In addition to the compression ratio, these features were used in decision trees and random forests for malware classification.

Others have gone to a more direct path of vectorization. Notice that a naive pairwise NCD matrix construction requires much redundant computation: for some fixed sample $x$, the computations $C(xy)$ are identical up to the moment they finish processing $x$. It would be more efficient to directly compare the compressor's learned model for each sequence. The \textit{Lempel-Ziv Jaccard Distance} (LZJD)~\cite{raff_lzjd_2017} is an approach to mapping sequences to a smaller \textit{set} representation which can be directly compared via Jaccard similarity. LZJD only implements the aspect of LZ compression which constructs the compression dictionaries $A$ and $B$ of two sequences $x$ and $y$, respectively, ignoring the extra steps involved in practical compression. Here, a compression dictionary is a set of unique subsequences. LZJD is defined as one minus the Jaccard similarity between the two sequences' compression dictionaries:
$\textnormal{LZJD}(x,y) = 1 - \frac{|A \cap B|}{|A \cup B|}$~\cite{raff_lzjd_2017,Raff2017MalwareCA}. LZJD has been used in sequence classification, with particular success in malware classification involving long byte sequences.
Similarly, the \textit{Burrows Wheeler Transform Markov Distance} (BWMD)~\cite{raff2020new} exploits the same idea of ``extracting'' the representation of the compressor to produce feature vectors, and instead applies the same trick to the Burrows Wheeler Transform (or BWT, as used in the \texttt{bzip2} compression program), which can be modeled as a Markov process since the BWT is meant to be used with run-length encoding.

\subsection{MDL-inspired prediction with compression models}
\label{sec:MDL}
A compressor can be used to classify even without invoking a formal notion of distance. Like NCD, approaches inspired by the Minimum Description Length principle have seen wide application to sequence prediction tasks such as text categorization~\cite{eibeTextCategorization,teahan2001compression,marton2005compression,bratko2006spam,ftcc}. Imagine we have a classification problem with classes $C_1, ..., C_m$, along with a compressor. Now consider, for each class $C_i$, concatenating together all of its samples into a single sequence $B_i$. With the intuition that popular dictionary-based compression algorithms learn a model of the input as they compress, compressing $B_i$ will result in the compression program building a model unique to the data in that class. Given a test point $x$, a simple classification scheme then is to compress $B_i x$ for each class, and predict the class which minimizes $C(B_i x) -  C(B_i)$. This predicted class is the one that is most similar to $x$ according to the compressor's model. In other words, it is the class whose learned model best compresses $x$. For Lempel-Ziv-style compressors, similarity will be due to matching byte sequences. We refer to this technique where class data and test data are aggregated and compressed together as implicit-model MDL, as the compression model is built in an online fashion allowing the compressor to learn patterns within the class data prior to compressing $x$.

The procedure is as follows: given a test sample $x\in \Sigma^*$, for each class $i$, construct the context $B_i$ from that class's training data by aggregating all (or some subset) of points from class $i$ aggregated according to a specified aggregation operation. The compression measure is the conditional code length
$\Delta_i(x)
  = M_{\text{MDL}}\!\big(C(B_i),\,C(B_i x)\big) =C(B_i x) - C(B_i)$.
Prediction is then given by
$\hat{y} = \arg\min_{i} \Delta_i(x)$.

We give pseudocode for the implicit-model MDL classification procedure in Algorithm \ref{alg:implicit-mdl}. The conditional code length is the compression cost of compressing $x$ when using the model learned from the class data $B_i$. This approach treats $B_i$ and $x$ not as peer samples but as a class model and a sample, respectively. This expression is rooted theoretically in conditional Kolmogorov complexity, approximating $K(x | B_i)$.

\begin{algorithm}[!h]
\caption{Implicit-Model MDL classification}
\label{alg:implicit-mdl}
\begin{algorithmic}
    \State Aggregate all training samples in each class into $B_j$
    \State Precompute $C(B_j)$ for each class $j$
    \For{each test sample $x_i$}
        \For{each class $j$}
            \State Compute $\Delta_{ij} \gets C(B_j x_i) - C(B_j)$
        \EndFor
        \State Predict $\hat{y}_i \gets \arg\min_j \Delta_{ij}$
    \EndFor
\end{algorithmic}
\end{algorithm}

Some compressors allow the explicit construction/initialization of compression models (e.g. dictionaries), which allow for an MDL-inspired classification technique which is able to cache the compression models, avoiding costly recomputation upon each new inference. For example, the \texttt{zstandard}~\cite{zstandard} compression algorithm allows for the explicit construction of compression dictionaries. Here, the classification procedure is to assign the class whose compression model best compresses a test point. We refer to this technique where a compression model is trained prior to inference as explicit-model MDL. In this case, the conditional code length is simply the compressed length when using the class-specific model and we denote this length $\Delta^*$:
$\Delta^{\star}_i(x)
  = M_{\text{MDL}}\!\big(C_{m_i}(x)\big) = C_{m_i}(x).$
We provide pseudocode for the explicit-model MDL classification procedure in Algorithm \ref{alg:explicit-mdl}.

\begin{algorithm}[H]
\caption{Explicit-Model MDL Classification}
\label{alg:explicit-mdl}
\begin{algorithmic}
    \State \textbf{Given:} $D_i$, the training data for class $i$
    \State For each class $i$, build a class-specific model $m_i$ from $D_i$
    \For{each test sample $x$}
        \For{each class $i$}
            \State $\Delta^{\star}_i(x) \gets C_{m_i}(x)$
        \EndFor
        \State Predict $\hat{y}(x) \gets \arg\min_i \Delta^{\star}_i(x)$
    \EndFor
\end{algorithmic}
\end{algorithm}

We compare runtime complexity of the approaches we have highlighted with respect to the number of compressor calls in Table \ref{tab:complexity}.

\begin{table}[h]
\caption{Computational cost of compression-based ML approaches; $n$ is the number of samples and $|\mathcal{Y}|$ the number of classes. LZJD/BWMD transform single samples into a compressed representation which are directly comparable via their respective dissimilarity measures.}
\centering
\renewcommand{\arraystretch}{1.2}
\begin{tabular}{@{}lcc@{}}
\toprule
\textbf{Approach} & \textbf{Compressor calls} & \textbf{Input per call} \\
\midrule
NCD $+$ $k$-NN        & $\mathcal{O}(n^2)$            & train $+$ test sample \\
LZJD / BWMD           & $\mathcal{O}(n)$             & one sample \\
Implicit-model MDL    & $\mathcal{O}(n|\mathcal{Y}|)$ & class context $+$ test sample \\
Explicit-model MDL    & $\mathcal{O}(n|\mathcal{Y}|)$ & test sample only \\
\bottomrule
\end{tabular}

\label{tab:complexity}
\end{table}

\subsection{Experimental Results}

We run a small set of representative experiments to both highlight the competitiveness of compression-based prediction methods and to show the high performance variability when selecting the $(C,S,M,P)$ design choices of our compression-based ML framework. We choose one dataset per domain: text (AGNews~\cite{zhang2015character}), malware (Drebin~\cite{arp2014drebin}), and images (MNIST~\cite{deng2012mnist}), and sweep across compression programs, context size, compression measure, aggregation operation, and prediction technique, organized around a standard NCD reference technique: \texttt{gzip} with NCD and $1$-NN, one context object per training sample, and string concatenation as the aggregation operation. For compressors, we use \texttt{gzip}, \texttt{bzip2}, \texttt{LZMA}, and \texttt{zstandard} at their default settings, and additionally run \texttt{zstandard} at compression level 15 (default is 3) to separate compression effort from choice of program. For compression measures we use NCD, CDM, CLM and
CosS, together with the conditional code length, and for $P$ we use $k$-NN with
$k \in \{1,3,5\}$, average distance to each class, and the MDL argmin rule. We report accuracy in both few-shot and full dataset settings, though due to the computational complexity of NCD, we only run it up to 1,000 training samples. Few-shot results are averaged over ten seeds and the $n{=}10^3$ setting over three seeds where each seed redraws the training samples, the evaluation subset, and for Drebin, the 90/10 train/test split. We
evaluate on $300$ test samples for AG News and MNIST and $100$ for Drebin, except in the full-corpus setting which uses the complete test sets. Each $\Delta$ is a difference of seed-averaged accuracies for a single fixed
configuration on each side; its standard error across seeds is at most $0.024$
in every setting. NCD/MDL are compared on identical splits against conventional baselines: logistic regression, a linear SVM, and 1-NN over TF-IDF word 1–2-grams (text), hashed byte 4-grams (malware), and raw pixels (images). For text and malware, compression methods use string concatenation for aggregation. For MNIST, aggregation sweeps across string concatenation and pixel-wise addition as in~\cite{jakobs_mnist_gzip}. To illustrate the variability of compression-based performance when sweeping across design choices, we show the accuracy span per choice in Figure \ref{fig:span}.

\begin{figure}[htbp]
    \centering
    \includegraphics[width=\columnwidth]{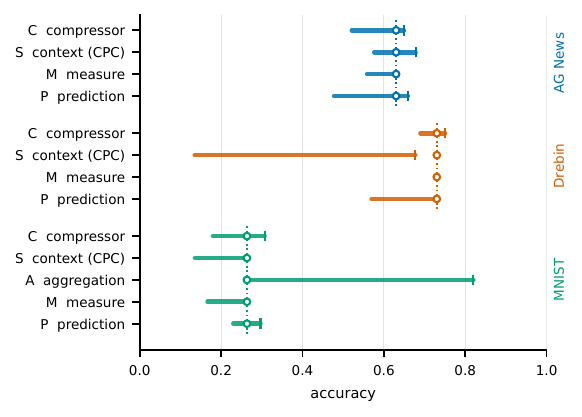}
    \caption{Accuracy span across our compression-based ML design choices. Performance is highly variable, and the degree to which varying one particular design choice matters is domain-dependent. }
    \label{fig:span}
\end{figure}

We report our results in Table \ref{tab:experiments}. NCD and MDL perform competitively in the few-shot setting, especially outperforming the baselines in few-shot malware classification, by as much as 0.27 points, indicating the high learning signal present in compression-based methods even with only a few samples. Taking \texttt{gzip} with NCD and $1$-NN as the reference configuration~\cite{jiang-etal-2023-low}, the
best configuration at the same training budget gains $0.013$ on Drebin, $0.170$
on AG News and $0.619$ on MNIST. Raising \texttt{zstandard} from level 3 to 15 improves accuracy in 82 of 110 configurations, by $0.016$ on average. Overall, MDL requires far fewer compressor calls than NCD while often also obtaining higher accuracy in the text and malware domains. Importantly, the image domain indicates the crucial role of aggregation, as aggregating via byte string concatenation does not aid in compression resulting in degenerate performance. However, modifying the aggregation operation to pixel-wise addition aids in compression of multiple samples and results in a non-trivial MNIST accuracy of up to 0.88 and competitive few-shot learning accuracies. NCD is not run on full MNIST due to computational constraints; at that scale, only MDL is affordable, however in the image domain a compression dictionary built from thousands of aggregated digits does not capture the class better than a single sample does. We therefore stress that a domain-aware aggregation operation is an important design choice worth considering.

\begin{table}[h]
\centering
\caption{Compression-based classification against standard feature-based classifiers in both few-shot and full dataset setting.  \textsc{comp.} and \textsc{base.} report their best configurations. \textsc{method} reports winning compression-based  method.  NCD not run in full dataset setting. Timings are single-threaded on an Apple M1 Max (10-core, 64 GB).}
\label{tab:experiments}
\footnotesize
\begin{tabular}{llcrrrr}
\toprule
Domain & Regime & \textsc{method} & Comp. & Base. & $\Delta$ & ms/query \\
\midrule
text & 1-shot & NCD & 0.344 & 0.314 & +0.030 & 0.96 \\
 & 5-shot & MDL & 0.426 & 0.436 & -0.011 & 0.08 \\
 & 10-shot & MDL & 0.513 & 0.509 & +0.004 & 0.17 \\
 & $n{=}10^3$ & MDL & 0.800 & 0.816 & -0.016 & 0.23 \\
 & full & MDL & 0.906 & 0.927 & -0.021 & 1.22 \\
\midrule
malware & 1-shot & MDL & \textbf{0.396} & 0.181 & \textbf{+0.215} & 17.93 \\
 & 5-shot & MDL & \textbf{0.702} & 0.436 & \textbf{+0.266} & 21.72 \\
 & 10-shot & MDL & \textbf{0.766} & 0.540 & \textbf{+0.226} & 23.14 \\
 & $n{=}10^3$ & NCD & \textbf{0.807} & 0.683 & \textbf{+0.123} & 17.34 \\
 & full & MDL & \textbf{0.936} & 0.852 & \textbf{+0.084} & 281.79 \\
\midrule
images & 1-shot & NCD & 0.434 & 0.432 & +0.001 & 2.68 \\
 & 5-shot & NCD & 0.656 & 0.684 & -0.028 & 12.33 \\
 & 10-shot & NCD & 0.733 & 0.755 & -0.021 & 24.24 \\
 & $n{=}10^3$ & NCD & 0.886 & 0.891 & -0.006 & 228.08 \\
 & full & MDL & 0.385 & 0.969 & -0.584 & 1.78 \\
\bottomrule
\end{tabular}
\vspace{2pt}
\end{table}

From our experimental results we conclude that compression-based ML essentially begs a hyperparameter optimization step across choice of compressor, compression measure, aggregator, and prediction technique, and there are immediate gains to be had over the conventional \texttt{gzip} + $k$-NN configuration popularized in~\cite{jiang-etal-2023-low}. NCD/MDL are especially competitive in the malware setting, where strong featurization techniques are lacking and Lempel-Ziv compression mechanisms essentially provide automatic feature extraction at variable-length n-grams.

\section{Open Questions and Challenges}
\label{sec:challenges}
The hypothesis that better compression yields better prediction is soundly supported when considering theoretically optimal compressors such as Kolmogorov complexity, but requires further empirical and theoretical justification when considering practical compressors. The success of machine learning in improving data compression is a strong signal that learned compressors may provide novel and useful avenues in solving sequential ML tasks via compression-based approaches. State of the art approaches for lossless text compression, for example via the Large Text Compression Benchmark~\cite{mahoney_large_text_compression}, are neural methods such as NNCP~\cite{bellard2021nncp} and CMIX~\cite{cmix}. And popular pre-trained large language models themselves can be used for lossless compression~\cite{deletang2024languagemodelingcompression}. Some initial work has been done in studying the effects of using neural compressors for compression-based ML~\cite{hurwitz2025llm_ncd,jiang_few-shot_2022}, however further study is warranted, especially as neural compressors, despite achieving superior compression, are currently much slower, hindering a larger scale study of their potential for compression-based prediction. Our framework separates the choice of compressor from its context mechanism, and our experiments show that simply modifying the aggregation operation moves MNIST accuracy from $0.27$ to $0.85$. In the neural compression setting, further study of alternative aggregation methods for NCD \& MDL is a promising avenue for realizing improvements, as was previously demonstrated in the malware domain~\cite{borbely2015normalizedcompressiondistancelarge}.

Techniques like LZJD and BWMD learn compressed representations of sequences for the goal of achieving good similarity scores, without the overhead required for decompression. One area of future work is to continue to improve or invent new compressed representations with the end goal of machine learning in mind, to improve accuracy, runtime, and memory use. 

Further, improvements in lossless compression itself have long been hypothesized to yield insights in artificial intelligence, as was the impetus for the Hutter Prize~\cite{Hutter:06hprize} and Mahoney's Large Text Compression Benchmark, and therefore remain an important target.

\section{Conclusion} \label{sec:conclusion}

Despite machine learning having produced new state-of-the-art in compression, and compression having been used for machine learning, the intersection of these two has yet to be fully understood. By formalizing compression-based ML into four key design choices, we find that the literature presents a refreshingly different approach to AI problems, still leaving a wide array of challenges to explore. 

\bibliographystyle{IEEEtran}
\bibliography{references}

\end{document}